\documentclass{article}
\usepackage[utf8]{inputenc}
\usepackage{geometry}

\usepackage{graphicx}
\usepackage{wrapfig}
\usepackage{threeparttable}
\usepackage{multirow}
\usepackage{subcaption, booktabs}
\usepackage[colorlinks=true]{hyperref}
\usepackage{color}
\usepackage{url}

\usepackage{amsfonts}
\usepackage{amsthm}
\usepackage{amssymb}
\usepackage{amsmath}

\usepackage{cleveref}

\newtheorem{theorem}{Theorem}[section]
\newtheorem{lemma}[theorem]{Lemma}

\title{Multifidelity linear regression\\ for scientific machine learning from scarce data}
\author{Elizabeth Qian\footnote{School of Aerospace Engineering, Georgia Institute of Technology, Atlanta, Georgia, USA}\,\,\footnote{School of Computational Science and Engineering, Georgia Institute of Technology, Atlanta, Georgia, USA}, Dayoung Kang\footnotemark[1], Vignesh Sella\footnote{Oden Institute for Computational Engineering and Sciences, University of Texas at Austin, Austin, Texas, USA}, Anirban Chaudhuri\footnotemark[3]}

\newcommand{\Var}{\operatorname{\mathbb{V}ar}}
\newcommand{\Cov}{\operatorname{\mathbb{C}ov}}
\newcommand{\Tr}{\mathsf{Tr}}

\newcommand{\R}{\mathbb{R}}
\newcommand{\E}{\mathbb{E}}

\begin{document}

\maketitle
\begin{abstract}
    Machine learning (ML) methods, which fit to data the parameters of a given parameterized model class, have garnered significant interest as potential methods for learning surrogate models for complex engineering systems for which traditional simulation is expensive. However, in many scientific and engineering settings, generating high-fidelity data on which to train ML models is expensive, and the available budget for generating training data is limited, so that high-fidelity training data are scarce. ML models trained on scarce data have high variance, resulting in poor expected generalization performance. 
    We propose a new \textit{multifidelity} training approach for scientific machine learning via linear regression that exploits the scientific context where data of varying fidelities and costs are available: for example, high-fidelity data may be generated by an expensive fully resolved physics simulation whereas lower-fidelity data may arise from a cheaper model based on simplifying assumptions. We use the multifidelity data within an approximate control variate framework to define new multifidelity Monte Carlo estimators for linear regression models. We provide bias and variance analysis of our new estimators that guarantee the approach's accuracy and improved robustness to scarce high-fidelity data. Numerical results demonstrate that our multifidelity training approach achieves similar accuracy to the standard high-fidelity only approach with orders-of-magnitude reduced high-fidelity data requirements.
\end{abstract}

\section{Introduction}\label{sec: intro}
Scientific and engineering decision-making rely on \textit{many-query} analyses, in which a predictive model must be evaluated many times at different inputs, initializations, or parameters. Examples include optimization, control, and uncertainty quantification. Traditional high-fidelity models of complex scientific and engineering systems are often prohibitively expensive for this many-query setting, necessitating the development and use of computationally efficient surrogate models. 
Machine learning (ML) methods, which fit to data the parameters of a given parametrized model class, can learn extremely complex functional relationships from data~\cite{hornik1989multilayer,lu2021learning}. A growing body of literature uses such methods to learn surrogate models for engineering systems, for example using linear and kernel regressions~\cite{peherstorfer16datadriven,qian2019liftlearn,qian2021reduced,mezic2005spectral,mezic2013analysis,williams2015data,williams2015kernel,nelsen2021random}, as well as nonlinear regressions, for example using neural networks~\cite{swischuk2018physics,lu2021learning,hesthaven2018non,bhattacharya2021model,li2020fourier,oleary2022derivative,luo2023efficient}. 
These works often take for granted the existence of (or ability to generate) a sufficiently large volume of high-fidelity training data for learning an accurate and robust model. This is a barrier to more widespread adoption of ML surrogate modeling methods in engineering and science because realistic budget constraints in these settings mean that high-fidelity data are usually scarce, due to the expense of obtaining data through expensive simulations or experiments. 
ML surrogate models trained on scarce data are sensitive to quirks of the data set and lead to less accurate and robust predictions~\cite{dehoop2022costaccuracy}, limiting trust in the learned models for use in high-consequence engineering and scientific application.

One way to address the data scarcity challenge is through \textit{multifidelity} scientific machine learning, which seeks to learn models from a combination of scarce high-fidelity data and more abundant lower-cost, lower-fidelity data, which are often available in scientific and engineering contexts. For example, while high-fidelity data may come from a high-resolution multi-physics simulation of a three-dimensional system, lower-fidelity data may be available from simulations that use lower resolutions or make simplifying physics assumptions. 
We first review approaches in which several separate (learned) models are combined to produce a single multifidelity model. These approaches include works that learn separate models from high- and low-fidelity data and them combine these models to issue multifidelity predictions~\cite{gerstner2021multilevel,guo2022multi}, as well as hybrid approaches which use high- and low-fidelity data to learn a correction or discrepancy for a provided physics-based low-fidelity model, and then combine the low-fidelity model with the learned correction model to issue multifidelity predictions~\cite{de2022bi,haftka1991combining,chang1993sensitivity,knill1999response,demo2023deeponet,ahmed2023multifidelity,moya2023bayesian,zhang2018multifidelity,fernandez2019linear}. 
One can also learn a low-fidelity model from low-fidelity data, and then learn a correction term: in some cases, the low-fidelity model is trained first and then the correction term is trained~\cite{liu2019multi,lu2022multifidelity,sella2023projection}; in other cases, the low-fidelity model is joined to the correction model and the training for both proceeds simultaneously in an `end-to-end' training process~\cite{howard2022multifidelity,ahmed2023multifidelity,meng2020composite}. 
These approaches which jointly train low-fidelity and correction models can also be viewed as belonging to another family of multifidelity machine learning methods in which data of varying fidelity are combined during the training process to learn the parameters of a single multifidelity model. These approaches include transfer learning methods, which first use lower-fidelity data to pre-train a model, and then use high-fidelity data to refine estimates of a subset of the model parameters~\cite{ramezankhani2022data,de2020transfer,jiang2023use,de2022neural,song2022transfer}, as well as data augmentation strategies in which high- and low-fidelity data sets are combined, potentially with different weights~\cite{sella2023projection}, or by mapping the high- and low-fidelity data to a shared feature space~\cite{sella2023projection,shi2020multi} in which the model is then learned. Bayesian learning methods provide another way to combine high- and low-fidelity data to train a single model: in multifidelity Gaussian process regression~\cite{brevault2020overview,poloczek2017multi,marques2018contour,chaudhuri2021mfegra} and co-kriging approaches~\cite{kennedy2000predicting,le2014recursive,parussini2017multi,perdikaris2017nonlinear}, the parameters of Gaussian process regression models are fit to multifidelity data by exploiting the structure of the correlation between the data of varying fidelities in a Bayesian parameter update. The works~\cite{gorodetsky2020mfnets,gorodetsky2021MFNetsDataEfficient} formulate a more general Bayesian multifidelity learning approach that can accommodate non-Gaussian distributions.

In this work, we propose a new multifidelity scientific machine learning approach for linear regression. We note that linear regression models, while comparatively simpler than nonlinear regression models such as neural networks, are a large class of models that can include models that are linear in \textit{features} that are arbitrarily nonlinear in the model \textit{inputs}, including features that are defined by neural networks~\cite{bolager2024sampling,nelsen2021random}. Our approach is based on multifidelity control variate methods for variance reduction of Monte Carlo estimators, which combine high-fidelity estimators computed using a small sample size with a correction term based on a larger set of low-fidelity samples. This correction term is formulated in a way that preserves unbiasedness of the control variate estimator with respect to the high-fidelity statistics, in contrast to some of the aforementioned correction approaches which do introduce bias.
Multifidelity control variate methods include Multilevel Monte Carlo (MLMC)~\cite{teckentrup2013further,giles2008multilevel,giles2015multilevel}, Multifidelity Monte Carlo (MFMC)~\cite{peherstorfer16datadriven,qian2018multifidelity}, generalized approximate control variates (ACV)~\cite{gorodetskyGeneralizedApproximateControl2020}, and Multilevel Best Linear Unbiased Estimators (MLBLUE)~\cite{schadenMultilevelBestLinear2020,schadenAsymptoticAnalysisMultilevel2021,crociMultioutputMultilevelBest2023,destouchesMultivariateExtensionsMultilevel2023}.
These approaches differ in the types of low-fidelity models they consider, and in the structure of their low-fidelity correction term(s). In contrast to multifidelity Bayesian inference methods like co-kriging, multifidelity control variate estimators are \textit{frequentist} methods, although connections between the two approaches have been shown~\cite{gorodetsky2020mfnets}.
Multifidelity control variate ideas have been previously used for machine learning in several different ways: the work~\cite{gerstner2021multilevel} learns a hierarchy of neural networks of different levels and then combines these neural networks in an MLMC control variate framework to issue predictions. The work~\cite{chada2022multilevel} uses multilevel sequential Monte Carlo sampling to compute the parameters of a deep neural network in a Bayesian way, while in~\cite{destouchesMultivariateExtensionsMultilevel2023}, the authors extend MLBLUE to the estimation of covariance matrices in the context of variational data assimilation problems. In the previously mentioned Bayesian multifidelity learning approaches~\cite{gorodetsky2020mfnets,gorodetsky2021MFNetsDataEfficient}, the multifidelity Bayesian posterior mean can be shown to be a control variate estimator. Another use of MLMC ideas in machine learning is in the control variate estimation of gradients in gradient-based optimization for training neural networks~\cite{shi2021multilevel,fujisawa2021multilevel} or optimization under uncertainty~\cite{de2022bi}. 

To the best of our knowledge, this work presents the first use of multifidelity control variates for surrogate model learning via linear regression. 
Our contributions are the following: 
\begin{enumerate}
    \item We propose new multifidelity control variate estimators for the unknown parameters of linear regression models. Our overall approach is generally applicable to the use of any multifidelity control variate estimator, including ACV and MLBLUE estimators, but we specifically propose new MFMC estimators for linear regression for clarity of exposition.
    \item We provide bias and variance analysis of the proposed MFMC estimators for linear regression that (i) guarantees unbiasedness of the multifidelity estimators for the unknown model parameters as well as the resulting multifidelity learned model predictions and (ii) informs optimal choices of hyperparameters for the method which improve the accuracy and robustness of our approach. 
    \item We conduct numerical experiments on both an analytical example as well as a convection-diffusion-reaction model problem. Our numerical results demonstrate that models learned from scarce high-fidelity data using our multifidelity approach have similar accuracy to standard learned models trained on orders-of-magnitude more high-fidelity data. 
\end{enumerate}

The remainder of this manuscript is organized as follows. \Cref{sec: background} formulates the linear regression problem from a statistical perspective and summarizes the multifidelity Monte Carlo approach. \Cref{sec: MF regression} presents our new multifidelity approach to linear regression, discusses choices of hyperparameters of the method, and provides analysis of method. \Cref{sec: numerics} demonstrates the efficacy of the method on two examples: an analytical model problem, and a convection-diffusion-reaction simulation. \Cref{sec: conclusions} concludes and provides a discussion of directions for future work.

\section{Background}\label{sec: background}
We introduce the linear regression problem we consider in \Cref{subsec: formulation}, followed by background on multifidelity control variate estimators in \Cref{subsec: mfmc bg}.

\subsection{Problem formulation}\label{subsec: formulation}
Let $Z\in\mathcal{Z}$ denote an input random variable with distribution $\pi$. We denote by $Y^{(1)} = f^{(1)}(Z)$ the scalar output random variable of interest, where $f^{(1)}:\mathcal{Z}\to\R$ denotes the true input-output relationship. Our focus is on settings where $f^{(1)}$ is expensive to evaluate, because it requires running an expensive computational code or conducting a physical experiment. 

We now define $x:\mathcal{Z}\to\R^d$ so that $X = x(Z)$ is a $d$-dimensional \textit{feature} random variable.
The goal of linear regression is to infer an approximate model of the form 
\begin{align}
    \hat f(z;\beta) = x(z)^\top\beta
\end{align}
by selecting $\beta\in\R^d$ such that $\hat f(Z;\beta)\approx Y^{(1)}$. 
Ideally, we would like to find the parameter $\beta$ that minimizes the expected square error over the entire probability distribution, as follows:
\begin{align}\label{eq: exact beta}
    \beta^* = \arg\min_{\beta\in\R^d}\E \frac12 \|\hat f(Z;\beta) - Y^{(1)}\|_2^2= \arg\min_{\beta\in\R^d}\E \frac12 \|X^\top\beta - Y^{(1)}\|_2^2.
\end{align}
The exact minimizer is given by
\begin{align}\label{eq: beta star}
    \beta^* = \E[XX^\top]^{-1} \E[XY^{(1)}].
\end{align}
If $X,Y$ have mean zero, then $\E[XX^\top] = \Cov[X,X]$ and $\E[XY^{(1)}] = \Cov[X,Y^{(1)}]$, and the solution~\eqref{eq: beta star} has
the interpretation that $\Cov[X^\top\beta^*-Y^{(1)},X]=0$, i.e., the residual must be linearly uncorrelated with the random inputs $X$.

In practice, computing exact expectations with respect to the measure $\pi$ is often not possible, so the expectations in~\eqref{eq: beta star} are estimated from data. The standard approach uses a training data set $\{(x_i,y_i^{(1)})\}_{i=1}^n$, where $y_i^{(1)} = f^{(1)}(z_i)$ and $x_i = x(z_i)$ for $z_i$ that are drawn i.i.d.\ from $\pi$. Denote by $X_n\in\R^{d\times n}$ the matrix whose $n$ columns are the feature data $x_1,\ldots,x_n\in\R^d$, and denote by $Y_n^{(1)}\in\R^n$ the vector whose elements are the output data $y_1^{(1)},\ldots,y_n^{(1)}$. Then, a standard approximation to~\eqref{eq: beta star} is given by
\begin{align}
    \tilde \beta^{(n)} = \left(\hat{C}^{(n)}_{XX}\right)^{-1} \hat c_{XY}^{(n)},
\end{align}
where $\hat C_{XX}^{(n)} = \frac1n X_n X_n^\top$ and $\hat c_{XY}^{(n)} = \frac1n X_n Y_n^{(1)}$ are $n$-sample estimates of the exact expected products $C_{XX}\equiv \E[XX^\top]$ and $C_{XY}\equiv\E[XY^{(1)}]$, respectively. However, although $\hat C_{XX}^{(n)}$ is an unbiased estimator for $C_{XX}$, $(\hat C_{XX}^{(n)})^{-1}$ inverse is only an \textit{asymptotically} unbiased estimator for $C_{XX}^{-1}$~\cite{hartlap2007your}, so $\tilde\beta^{(n)}$ is a biased estimator for $\beta^*$, particularly in the scarce data setting when $n$ is small. 

In this work, we consider the setting where evaluating the model $f^{(1)}$ to compute the output data $y_i^{(1)} = f^{(1)}(z_i)$ is expensive. Crucially, the feature data $x_i$ do not require evaluation of the expensive model $f^{(1)}$, and thus the cost of acquiring very good estimates of $C_{XX}$ can be considered negligible compared to the cost to compute the estimate $\hat c_{XY}^{(n)}$. In fact, in many engineering settings the distributions for the input variables $z$ are treated as uniform distributions and thus exact expressions for $C_{XX}$ can be obtained. Our focus in this work is therefore on the following alternative approximation to~\eqref{eq: beta star}:
\begin{align}\label{eq: CE beta}
    \hat \beta^{(n)} = (C_{XX})^{-1}\hat c_{XY}^{(n)},
\end{align}
where we assume we have the ability to compute $C_{XX}=\E[XX^\top]$ either exactly or with very many ($N\gg n$) input samples, but have only limited (small $n$) output samples with which to compute $\hat c_{XY}^{(n)}=\frac1n X_nY_n^{(1)}$. The estimator in \eqref{eq: CE beta} is an unbiased estimator for $\beta^*$, even when $n$ is small, and our analysis in~\Cref{subsec: analysis of MFCE} will rely on this fact.

\textbf{Randomness of regression solution and non-robustness to scarce data.}
The estimated regression coefficients in \eqref{eq: CE beta} are themselves random variables whose realizations depend on the training data. In scientific settings, where computational resources are limited and evaluating $f^{(1)}$ to obtain training data is expensive, we are often forced to work with small sample sizes $n$, which can lead to the estimated regression coefficients having high variance. Consider, for a fixed $z$, the conditional variance of the model:
\begin{align}\label{eq: conditional variance}
    \Var[\hat f(z;\hat\beta^{(n)})|Z=z] = x(z)^\top \Cov[\hat\beta^{(n)},\hat\beta^{(n)}]\, x(z).
\end{align}
Since $\hat\beta^{(n)}$ estimated from scarce data (small $n$) has high variance, estimates from scarce data are likely to be far from the $\beta^*$ given by~\eqref{eq: exact beta}. From \cref{eq: conditional variance}, we see that this in turn leads to learned model predictions that are less accurate in expectation because they are likely to be far from predictions of the ideal model defined by $\beta^*$. 
The goal of this paper is to introduce a multifidelity training strategy that uses data from both cheap low-fidelity models and the expensive high-fidelity model to reduce the variance of the estimated regression coefficients $\hat\beta$ while guaranteeing unbiasedness of the estimators. This in turn will lead to learned model predictions with lower variance and thus higher expected accuracy.

\subsection{Multifidelity Control Variate Estimators}\label{subsec: mfmc bg}
We now introduce the multifidelity control variate approach to reducing estimator variance. Recall that the standard Monte Carlo estimator of $\E[f^{(1)}(Z)]$ is given by 
\begin{align*}
    \hat \mu_n = \frac1n\sum_{i=1}^n f^{(1)}(z_i),
\end{align*}
where $\{z_i\}_{i=1}^n$ are i.i.d.\ realizations of the input random variable $Z$. This mean estimator $\hat\mu_n$ is unbiased, i.e., $\E[\hat\mu_n]=\E[f^{(1)}(Z)]$, and its variance is $\Var[\hat\mu_n] = \frac1n\Var[f^{(1)}(Z)]$. 

A control variate is another random variable $U$ with mean $\E[U] = \mu_u$, which defines the following control variate estimator:
\begin{align}\label{eq: control varaite}
    \hat\mu_n^{\rm CV} = \frac1n\sum_{i=1}^n f^{(1)}(z_i) + \alpha\left(\mu_u - \frac1n\sum_{i=1}^n u_i\right),
\end{align}
where $\alpha\in\R$ is called the control variate coefficient and $u$ is a realization of $U$. If $\mu_u$ is known exactly, the formula~\eqref{eq: control varaite} defines an \textit{exact} control variate. If we replace $\mu_u$ in~\eqref{eq: control varaite} with a sample estimate, then the resulting estimator is an \textit{approximate} control variate.
Both the exact and approximate estimators~\eqref{eq: control varaite} are unbiased, $\E[\hat\mu_n^{\rm CV}] = \E[f^{(1)}(Z)]$. The variance of the exact control variate estimator is given by
\begin{align*}
    \Var[\hat\mu_n^{\rm CV}] = \frac{\Var[f^{(1)}(Z)]}{n} + \alpha^2\frac{\Var[U]}{n} - 2\alpha \frac{\Cov[f^{(1)}(Z),U]}{n}.
\end{align*}
We can set the derivative of the above expression with respect to $\alpha$ to zero to obtain the optimal $\alpha^* = \arg\min\Var[\hat\mu_n^{\rm CV}]$:
\begin{align*}
    \alpha^* = \frac{\Cov[f^{(1)}(Z),U]}{\Var[U]}.
\end{align*}
With this choice of $\alpha$, the variance of the exact control variate estimator is given by 
\begin{align*}
    \Var[\hat\mu_n^{\rm CV}] = \frac1n \left(\Var[f^{(1)}(Z)] - \frac{\Cov[f^{(1)}(Z),U]^2}{\Var[U]}\right).
\end{align*}
Thus, the exact control variate estimator has lower variance than the standard Monte Carlo estimator if the control variate $U$ is correlated with $f^{(1)}(Z)$: the better the correlation, the greater the variance reduction. Approximate control variate estimators have more complex variance expressions and associated variance reduction conditions, but the main idea that improved correlation leads to greater variance reduction remains the same, see e.g.,~\cite{Peherstorfer15Multi,gorodetskyGeneralizedApproximateControl2020} for details. 

The main idea of multifidelity control variate estimation is to use low-fidelity models to define approximate control variates, i.e., let $f^{(2)}:\R^d\to\R$ be a lower-fidelity model for the same system described by $f^{(1)}$, and let the control variate be given by $U = f^{(2)}(Z)$. 
We can generalize this idea to a hierarchy of $K\geq 2$ models $f^{(1)},\ldots,f^{(K)}:\R^d\to\R$ by adding additional approximate control variates based on each model to the estimator~\eqref{eq: control varaite}. All multifidelity control variate strategies are variations on this theme:
for clarity of exposition, we now introduce only the multifidelity Monte Carlo (MFMC)~\cite{Peherstorfer15Multi} approach. MFMC is the focus of \Cref{sec: MF regression}, which introduces our method, and \Cref{sec: numerics}, which provides numerics. We emphasize that other multifidelity control variate strategies, such as ACV~\cite{gorodetskyGeneralizedApproximateControl2020} and MLBLUE~\cite{schadenMultilevelBestLinear2020,crociMultioutputMultilevelBest2023} estimators, can be substituted everywhere an MFMC estimator is used, and we discuss key differences between MFMC and ACV and MLBLUE where appropriate.

Recall that we consider $K\geq 2$ models $f^{(1)},\ldots,f^{(K)}:\R^d\to\R$. Let $w_k$ denote the cost of model $f^{(k)}$ and let 
\begin{align}
    \rho_{1,k}=\frac{\Cov[f^{(1)}(Z),f^{(k)}(Z)]}{\sigma_1\sigma_k}, \qquad \text{where}\quad  \sigma_k=\sqrt{\Var[f^{(k)}(Z)]}\,,
\end{align}
denote the Pearson correlation coefficient of $f^{(k)}(Z)$ with $f^{(1)}(Z)$. 
We assume $f^{(1)}$ is the high-fidelity reference model and models $f^{(2)},\ldots,f^{(K)}$ are lower-fidelity models of decreasing fidelity and decreasing cost: i.e., $w_1>w_2> \cdots >w_K$ and $1 = \rho_{11} > \rho_{12} > \cdots > \rho_{1K}$.

The MFMC estimator of $\E[f^{(1)}(Z)]$ is given by~\cite{Peherstorfer15Multi}
\begin{align}\label{eq: MFMC mean}
    \hat \mu^{\rm MF} = \frac1{m_1}\sum_{i=1}^{m_1}f^{(1)}(z_i) + \sum_{k=2}^K\alpha_k\left(\frac1{m_k}\sum_{i=1}^{m_k}f^{(k)}(z_i) - \frac1{m_{k-1}}\sum_{i=1}^{m_{k-1}}f^{(k)}(z_i)\right),
\end{align}
where $m_1<m_2<\cdots<m_K$ are the number of model evaluations for each of the $K$ models, and $\{z_i\}_{i=1}^{m_K}$ are i.i.d.\ draws of the input random variable as before.
One way to interpret~\eqref{eq: MFMC mean} is to view the control variate based on the $k$-th lower-fidelity model as a higher-sample correction to the estimator based on fewer samples of the 1st-through-$(k-1)$-th higher-fidelity models. 

The MFMC estimator is unbiased, i.e., $\E[\hat\mu^{\rm MF}] = \E[f^{(1)}(Z)]$,
and the MFMC estimator has a lower variance than the standard Monte Carlo estimator of the same cost, i.e., if $\sum_{k=1}^K m_kw_k= nw_1$, then
\begin{align*}
    \Var[\hat\mu^{\rm MF}] \leq \Var[\hat \mu_n],
\end{align*}
provided that the models satisfy certain conditions on their relative costs and correlations: loosely speaking, low-fidelity models must have sufficiently low cost to justify their low correlation (see~\cite[Corollary 3.5]{Peherstorfer15Multi} for details).  The work~\cite{Peherstorfer15Multi} provides a model selection algorithm that selects models so that these conditions are satisfied. Additionally, the work~\cite{Peherstorfer15Multi} shows that the optimal choice of the coefficients $\alpha_k$ is given by $\alpha_k = \frac{\rho_{1k}\sigma_1}{\sigma_k}$, and also provides an optimal allocation of a limited computational budget among the $K$ different models that minimizes $\Var[\hat\mu^{\rm MF}]$: that is, the optimal assignment of $m_1,\ldots,m_K$ such that $\sum_{k=1}^K m_k w_k =p$, where $p$ is the computational budget. This optimal assignment has an analytical solution and is given by
\begin{align}\label{eq: MFMC optimal allocation}
    m_1 = \frac{p}{w^\top r}, \quad \text{and} \quad m_k=m_1r_k\quad \text{for }k=2,\ldots,K,\quad\text{where}\quad r_k = \sqrt{\frac{w_1(\rho_{1,k}^2 - \rho_{1,k+1}^2)}{w_k(1-\rho_{1,k}^2)}},
\end{align}
where we set $\rho^2_{1,K+1}=0$.

Both the optimal sample allocation $m_1,\ldots,m_K$, and the optimal $\alpha_k$ depend on second-order statistics of the models $\sigma_k$ and $\rho_{1k}$. In practice, since these statistics are generally not known, they are estimated from pilot samples in order to determine the model allocation and coefficients $\alpha_k$, leading to sub-optimal $\alpha_k$ and $m_k$ choices that nevertheless have led to orders-of-magnitude cost reductions in multiple real-world applications, including subsurface resource management~\cite{mehana2023prediction}, space telescope engineering~\cite{cataldo2022multifidelity}, and topology and aircraft design optimization~\cite{chaudhuri2018multifidelity,hyun2023multifidelity}. In these applications, low-fidelity models with correlation as low as 38\% have been successfully used to reduce the variance of Monte Carlo estimators~\cite{cataldo2022multifidelity}.

One key difference between MFMC and the more general ACV and MLBLUE estimation strategies is in the sampling strategy: note that the MFMC estimator~\eqref{eq: MFMC mean} has a nested sampling strategy where the first $m_k$ inputs are evaluated at all models $f^{(1)},\ldots,f^{(k)}$. The ACV and MLBLUE estimators remove this constraint on the sampling strategy, and can achieve greater variance reductions than MFMC as a result. However, the optimal sampling strategy for both ACV and MLBLUE estimators still depends on second-order statistics for the models, and generally does not have an analytical solution and is instead obtained through numerical optimization methods, see~\cite{gorodetskyGeneralizedApproximateControl2020,schadenMultilevelBestLinear2020,crociMultioutputMultilevelBest2023} for details.

\section{Multifidelity linear regression}\label{sec: MF regression}
In this section, we introduce a multifidelity linear regression approach based on using multifidelity control variate estimators for the linear regression estimator~\eqref{eq: CE beta} rather than the standard high-fidelity only Monte Carlo estimate.
\Cref{subsec: MF estimators} presents our multifidelity linear regression setting and introduces our multifidelity linear regression strategy. \Cref{subsec: allocation and alpha discussion} presents and discusses several strategies for choosing the model allocations $m_k$ and the control variate coefficients $\alpha_k$ in the multifidelity covariance estimator. \Cref{subsec: analysis of MFCE} provides optimality analysis of some choices for the control variate coefficients. 

\subsection{Multifidelity estimators}\label{subsec: MF estimators}
This section introduces our multifidelity linear regression approach: we will introduce the estimator for the MFMC case of $K$ models $f^{(1)},\ldots,f^{(K)}$ with decreasing costs $w_1>w_2>\cdots w_K$ and decreasing fidelity $\rho_{11}> \rho_{12}> \cdots \rho_{1K}$, as before. 
We again emphasize that we use the MFMC estimator here to simplify the exposition, and note that our main idea applies equally well to other multifidelity control variate estimators like ACV and MLBLUE~\cite{gorodetskyGeneralizedApproximateControl2020,schadenMultilevelBestLinear2020}.

We assume that we have $m_1$ high-fidelity data pairs, $\{(x_i,y_i^{(1)})\}_{i=1}^{m_1}$, and $m_2> m_1$ low-fidelity data, $\{(x_i,y_i^{(2)})\}_{i=1}^{m_2}$, and so on, i.e., $m_k>m_{k-1}$ data $\{(x_i,y_i^{(k)}\}_{i=1}^{m_k}$ where $y_i^{(k)} = f^{(k)}(z_i)$ and $x_i = x(z_i)$. 
We assume that $C_{XX}$ is available (nearly) exactly, as discussed in~\Cref{subsec: formulation}. 
Let $X_{m_k}\in\R^{d\times m_k}$ denote the matrix whose columns consist of the first $m_k$ inputs, and let $Y_{m_k}^{(j)}\in\R^{m_k}$ for $j\leq k$ denote the vector whose elements are given by the first $m_k$ outputs of model $j$.
Let $A_k\in\R^{d\times d}$ for $k=2,\ldots,K$, and define the following multifidelity estimate of $c_{XY}$:
\begin{align}\label{eq: cov estimate A}
    \hat c_{XY}^{\rm MF} = \frac1{m_1}X_{m_1} Y^{(1)}_{m_1} + \sum_{k=2}^K A_k \left(\frac1{m_k}X_{m_k} Y^{(k)}_{m_k} - \frac1{m_{k-1}}X_{m_{k-1}} Y^{(k)}_{m_{k-1}}\right),
\end{align}
which defines a multifidelity estimate for the linear regression model:
\begin{align}
    \label{eq: MF CEA}
    \hat f(z;\hat\beta^{\rm MF}) = x(z)^\top\hat\beta^{\rm MF}, \qquad \text{where }
    \hat \beta^\text{MF} = (C_{XX})^{-1} \hat c_{XY}^{\rm MF}.
\end{align}
We note that it is equally possible to directly define a multifidelity control variate estimator for the unknown model parameters $\hat\beta$ or for the learned model $\hat f(z;\hat\beta)$ itself: these estimators are closely related to the approach defined in~\cref{eq: cov estimate A,eq: MF CEA}, which we discuss in more detail in \Cref{app: other CV estimators}. For now, we focus solely on the estimators defined by~\cref{eq: cov estimate A,eq: MF CEA}.

We now show (a) that the estimator $\hat c_{XY}^{\rm MF}$ is an unbiased estimator for $C_{XY}$, (b) that the estimator $\hat\beta_{\text{MF}}$ is an unbiased estimator for $\beta^*$, and (c) that the multifidelity learned models $\hat{f}(z;\hat\beta_{\text{MF}})$ issues unbiased predictions with respect to the theoretical optimal model $\hat{f}(z;\beta^*)$. 

\begin{theorem}\label{thm: unbiased} Unbiasedness of multifidelity linear regression approach:
    \begin{enumerate}
        \item $\E[\hat c_{XY}^{\rm MF}] =\E[XY^{(1)}]$,
        \item $\E[\hat \beta^{\rm MF}] = \beta^*$, and 
        \item $\E[\hat f(z;\hat \beta^{\rm MF})] = \hat f(z;\beta^*)$. 
    \end{enumerate} 
\end{theorem}

\begin{proof}
    The first statement follows from the definition of $\hat c_{XY}^{\rm MF}$ and the linearity of expectations:
    \begin{align*}
        \E[\hat c_{XY}^{\rm MF}] &= \E\left[\frac1{m_1}X_{m_1} Y^{(1)}_{m_1}\right] + \sum_{k=2}^K A_k \left(\E\left[\frac1{m_k}X_{m_k}Y^{(k)}_{m_k}\right] - \E\left[\frac1{m_{k-1}}X_{m_{k-1}} Y^{(k)}_{m_{k-1}}\right]\right)\\
        &=\E[XY^{(1)}] + \sum_{k=2}^K A_k (\E[XY^{(k)}] - \E[XY^{(k)}]) = \E[XY^{(1)}].
    \end{align*}
    The second statement now follows because $\hat\beta^{\rm MF}$ defined in~\eqref{eq: MF CEA} is linear in the estimate $\hat c_{XY}^{\rm MF}$. The third statement then follows from linearity of the model in $\hat \beta^{\rm MF}$.
\end{proof}

\textit{Remark.} We emphasize again that we consider settings in which $C_{XX}$ is available exactly. If $C_{XX}$ in~\cref{eq: MF CEA} is instead estimated from samples, then the first statement in~\Cref{thm: unbiased} remains true, but the second and third statements of~\Cref{thm: unbiased} no longer hold exactly. 

\subsection{Choices of multifidelity linear regression hyperparameters}\label{subsec: allocation and alpha discussion}
We note that our multifidelity regression approach depends on several hyperparameters: the numbers of high- and low-fidelity samples $m_k$, as well as the control variate coefficients $A_k\in\R^{d\times d}$ for $k = 1,\ldots,K$. This section presents several strategies for setting these hyperparameters and summarizes our recommendations: \Cref{subsec: CV coefficients} discusses choices of control variate coefficients and \Cref{subsec: allocation} discusses the sample allocations for the MFMC-type estimator introduced in~\Cref{subsec: MF estimators}.  Detailed analytical justifications for our recommendations are deferred to~\Cref{subsec: analysis of MFCE}. \Cref{sssec: allocation coefficient other CVs} briefly discusses choices of hyperparameters in alternative multifidelity control variate estimation strategies.

\subsubsection{Control variate coefficients}\label{subsec: CV coefficients}
We consider three different choices for the control variate coefficient $A_k$ in \eqref{eq: cov estimate A}. First, we propose the following heuristic choice:
\begin{align}\label{eq: mfmc alpha}
    \textbf{Heuristic choice:} \qquad A_k = \alpha_k^{\rm mean}I, \qquad \text{where } \alpha_k^{\rm mean} = \frac{\rho_{1,k}\sigma_1}{\sigma_k} = \frac{\Cov[f^{(1)}(Z),f^{(k)}(Z)]}{\Var[f^{(k)}(Z)]}.
\end{align}
As we described in~\Cref{sec: background}, the work~\cite{Peherstorfer15Multi} proves that $\alpha_k^{\rm mean}$ minimizes the MSE of the scalar MFMC mean estimator~\eqref{eq: MFMC mean} --- but this choice is not optimal for minimizing the variance of our $\hat c_{XY}$ estimator. Despite its sub-optimality, our numerical results show that this heuristic strategy nevertheless yields practical gains.

To define an optimal choice of $A_k$ for our multifidelity linear regression estimators, note that our goal is to estimate $C_{XY} = \E[XY^{(1)}] = \E[X\,f^{(1)}(Z)]$. Let $g^{(k)}(z) = x(z)\,f^{(k)}(z)$, and define 
\begin{align}\label{eq: C definitions}
    \Gamma_{1k} = \Cov[g^{(1)}(Z),g^{(k)}(Z)]), \quad \Gamma_{kk} = \Cov[g^{(k)}(Z),g^{(k)}(Z)]).
\end{align}
In~\Cref{subsec: analysis of MFCE}, we will show (\Cref{thm: optimal matrix alpha}) that the choice 
\begin{align}\label{eq: optimal matrix alpha}
    \textbf{Optimal matrix choice: }\qquad A_k^* = \Gamma_{1k} \Gamma_{kk}^{-1}
\end{align}
minimizes the generalized variance of~\eqref{eq: MF CEA}. \Cref{thm: optimal matrix alpha} in~\Cref{subsec: analysis of MFCE} additionally shows that this in turn minimizes both the variance of the learned model predictions, $\Var[\hat f(Z;\hat\beta^{\rm MF})|Z=z]$ and the generalized variance of the coefficients $\hat\beta^{\rm MF}$. However, we note that the optimal choice~\eqref{eq: optimal matrix alpha} depends on the second-order statistics $\Gamma_{1k}$ and $\Gamma_{kk}$, which generally are unknown and must be estimated from samples. 
Because $\Gamma_{1k}$ and $\Gamma_{kk}$ in~\eqref{eq: C definitions} are $d\times d$ covariance matrices, obtaining good estimates of these statistics generally will require many more pilot samples than estimating the scalar statistics $\rho_{1,k}$, $\sigma_1$, and $\sigma_k$ which define~\eqref{eq: mfmc alpha}. For this reason, we also consider restricting $A_k$ to be a scalar multiple of the identity: $A_k = \alpha_k I$ (similar to the heuristic choice), i.e., essentially using a scalar coefficient $\alpha_k\in\R$ instead of the matrix coefficient $A_k\in\R^{d\times d}$. We will show in \Cref{thm: optimal scalar alpha} that the choice
\begin{align}\label{eq: optimal scalar alpha}
    \textbf{Optimal scalar choice: }\qquad A_k = \alpha_k^*I, \qquad \text{where }\alpha_k^* = \frac{\Tr (\Gamma_{1k})}{\Tr (\Gamma_{kk})}
\end{align}
minimizes the generalized variance (trace of the covariance) of the multifidelity estimator over all possible scalar coefficients $\alpha_k$. We will show in~\Cref{subsec: analysis of MFCE} that this in turn minimizes an upper bound on the conditional variance of the learned model predictions, $\Var[\hat f(Z;\hat\beta^{\rm MF})|Z=z]$. Note that~\eqref{eq: optimal scalar alpha} depends only on the traces of $\Gamma_{1k}$ and $\Gamma_{kk}$, which are scalar quantities which may be easier to estimate from pilot samples than the full covariance matrices $\Gamma_{1k}$ and $\Gamma_{kk}$ (and our numerical results in \Cref{sec: numerics} support this claim).

We will compare the three choices for $A_k$ described above in our numerical experiments in \Cref{sec: numerics} for both the case where model statistics are known and the case when they must be estimated from pilot samples. Our experiments will illustrate that the optimal matrix choice~\eqref{eq: optimal matrix alpha} leads to the lowest variance when the statistics $\Gamma_{1k}$ and $\Gamma_{kk}$ are known, but that this choice can perform poorly in settings when these covariances must be estimated, which is typically the case in practical applications. The optimal scalar choice~\eqref{eq: optimal scalar alpha} is more robust to sample estimates of $\Gamma_{1k}$ and $\Gamma_{kk}$ in our experiments, and is our general recommendation. However, our numerics also show that the sub-optimal heuristic choice~\eqref{eq: mfmc alpha} still consistently leads to significant variance reduction relative to the high-fidelity only case, and this choice may be preferred when the model statistics $\rho_{1k}$ and $\sigma_k$ are known or can be estimated more easily than $\Tr(\Gamma_{1k})$ and $\Tr(\Gamma_{kk})$.

\subsubsection{Sample allocation}\label{subsec: allocation}
There are two scenarios in which our multifidelity linear regression approach can be employed: in the first, the $m_1$ high-fidelity data and $m_k$ low-fidelity data for $k =2,\ldots,K$ are fixed and given, for example if studies of the model output's dependence on the input variables have already been run and their results saved. In this case, we then choose one of the $A_k$ choices discussed in~\Cref{subsec: CV coefficients} and implement the multifidelity $\beta$ estimator with the given sample allocation $m_k$.

In the second scenario, the user starts with no data but can choose to run the models $f^{(1)},\ldots,f^{(K)}$ at different inputs to generate output data $y_i^{(k)}$, subject to constraints on the overall cost of generating the output data. The allocation given by~\cref{eq: MFMC optimal allocation} for MFMC mean estimation from~\cite{Peherstorfer15Multi} minimizes the mean squared error of the MFMC estimator. This raises the question of how to adapt the optimal sample allocation approach from MFMC mean estimation to multifidelity linear regression. A natural way to do this is to minimize the generalized variance of the vector multifidelity estimator for $c_{XY}\in\R^d$. As we will show in~\Cref{subsec: analysis of MFCE}, this would also minimize the variance of the learned model predictions, conditioned on the input variable $z$. However, just as the MSE-minimizing sample allocation for MFMC mean estimation in~\eqref{eq: MFMC optimal allocation} depends on second-order statistics of the scalars $Y^{(k)} = f^{(k)}(Z)$ which must often be estimated in practice, any attempt to minimize the generalized variance of the vector MFMC estimator~\eqref{eq: MF CEA} will depend on second-order statistics of the vectors $XY^{(k)}$ which must be estimated. Unfortunately, these second-order statistics are then given by $d\times d$ covariance matrices, which require more samples to estimate well than the second-order scalar statistics required for~\cref{eq: MFMC optimal allocation}. Because poor estimates of these statistics will lead to sub-optimal sample allocations anyway, instead of developing an analogue of~\eqref{eq: MFMC optimal allocation} for the multifidelity linear regression problem, we instead propose the heuristic strategy of using~\eqref{eq: MFMC optimal allocation} as-is to determine the sample allocation for multifidelity linear regression. We will show in our numerical results that despite not being tailored to the multifidelity linear regression problem, the strategy~\eqref{eq: MFMC optimal allocation} leads to significantly more robust and accurate models learned from scarce data than standard models trained from high-fidelity only. 

To summarize, there are two different ways we recommend choosing the sample allocation $m_k$ in the estimator definitions in~\cref{eq: MF CEA}, depending on the usage scenario:
\begin{itemize}
    \item \textbf{Option 1: when a multifidelity data set is already available,} use all available data, where $m_1,\ldots,m_K$ are determined by how many samples are available in the data set, and 
    \item \textbf{Option 2: when no data are available,} use~\eqref{eq: MFMC optimal allocation} to determine how to generate data with a computational budget of $p$.
\end{itemize}
It is perhaps interesting to consider the question of a hybrid approach, i.e., when some data are already available but the user has the option to generate more data subject to some constraints. A strategy for tackling this problem could build on recent work in context-aware sampling for multifidelity uncertainty quantification~\cite{alsup2023context,farcas2023context}, but this is a direction left for future work. 

\subsubsection{Sample allocation and coefficient selection for other multifidelity control variate estimators}\label{sssec: allocation coefficient other CVs}
In \Cref{subsec: MF estimators} we have specifically introduced MFMC-type multifidelity control variate estimators for the linear regression problem, and our discussion of hyperparameter choices in \Cref{subsec: CV coefficients,subsec: allocation} and supporting analysis in \Cref{subsec: analysis of MFCE} is based on MFMC-type analysis of optimal control variate coefficients and sample allocations. We emphasize that our main idea of using multifidelity control variate estimators for linear regression could be equally well applied to estimators such as generalized approximate control variate (ACV) estimators~\cite{gorodetskyGeneralizedApproximateControl2020} and multilevel best linear unbiased estimators (MLBLUE)~\cite{schadenMultilevelBestLinear2020,crociMultioutputMultilevelBest2023}. In that case, the optimal choice of hyperparameters for the multifidelity linear regression would follow established results for optimal coefficient selection and sampling strategies for ACV~\cite{gorodetskyGeneralizedApproximateControl2020} and MLBLUE~\cite{schadenMultilevelBestLinear2020,crociMultioutputMultilevelBest2023}. Similar to the MFMC case, these hyperparameter choices would still depend on second-order statistics for the functions $g^{(k)}(Z)$: in fact, while MFMC depends only on $\Gamma_{1k}$ and $\Gamma_{kk}$ for $k\leq K$, both ACV and MLBLUE would require $\Gamma_{jk}$ for all$j,k\leq K$, which are in general unknown and must be estimated. 
Additionally, the optimal sample allocations for MLBLUE and ACV must be computed by solving a numerical optimization, in contrast to the analytical solution that exists for MFMC. 
To simplify the narrative, we have chosen to focus on MFMC-type estimators in our numerics and analysis, but detailed numerical and theoretical investigation of ACV and MLBLUE estimators for linear regression are areas for future work.


\subsection{Optimality analysis for control variate coefficients}\label{subsec: analysis of MFCE}
This section proves the optimality of the control variate coefficient definitions in~\eqref{eq: optimal matrix alpha} and~\eqref{eq: optimal scalar alpha}. To simplify the exposition, the analytical results are proved for the bi-fidelity case where $K =2$ and we drop the $\cdot_k$ subscripts for $A$ and $\alpha$, but the generalization of the results to $K>2$ case is straightforward.

The goal is to make our multifidelity learned model robust to random realizations of the training data: we thus begin by relating the variance of learned model predictions, conditioned on the given input $z$, to the covariance of the estimator $\hat c_{XY}^{\rm MF}$.
\begin{lemma}\label{lem: variance upper bound}
    For $i = 1,\ldots,d$, let $(\lambda_i,v_i)$ denote the eigenpairs of $\Cov[\hat c_{XY}^{\rm MF}, \hat c_{XY}^{\rm MF}]$. Then,
    \begin{align}\label{eq: var upper bound}
        \Var[\hat f(Z;\hat\beta^{\rm MF})|Z=z] = \sum_{i=1}^d\lambda_i c_i \leq \|c\|_2\Tr(\Cov[\hat c_{XY}^{\rm MF}, \hat c_{XY}^{\rm MF}]),
    \end{align}
    where $c = [c_1,\ldots,c_d]^\top\in\R^d$ is a constant vector dependent on $z$ but independent of the estimator $\hat\beta^{\rm MF}$.
\end{lemma}
\begin{proof}
    Let $\tilde x = C_{XX}^{-1}x(z)$ and note that 
    \begin{align*}
        \Var[\hat f(Z;\hat\beta^{\rm MF})|Z=z] &= \tilde x^\top \Cov[\hat c_{XY}^{\rm MF}, \hat c_{XY}^{\rm MF}]\tilde x= \sum_{i=1}^d \lambda_i \left<\tilde x,v_i\right>^2. 
    \end{align*}
    The equality in~\eqref{eq: var upper bound} follows from the assignment $c_i = \left<\tilde x,v_i\right>^2$. Then, note that $\sum_{i=1}^d \lambda_i c_i\leq \|\lambda\|_2\|c\|_2\leq \|\lambda\|_1\|c\|_2$. The inequality follows from the fact that $\|\lambda\|_1 = \Tr(\Cov[\hat c_{XY}^{\rm MF}, \hat c_{XY}^{\rm MF}])$.
\end{proof}

To minimize the variance of the learned model predictions, we therefore seek to minimize the eigenvalues of $\Cov[\hat c_{XY}^{\rm MF}, \hat c_{XY}^{\rm MF}]$. 
Let $g^{(k)}(z) = x(z)f^{(k)}(z)$ as before and note that $\hat c_{XY}^{\rm MF}$ can be rewritten as a multifidelity estimator for the mean of $g^{(1)}(z)$:
\begin{align*}\label{eq: MF vec A}
    \hat c_{XY}^{\rm MF} = \frac1{m_1}\sum_{i=1}^{m_1}g^{(1)}(z_i) + A\left(\frac1{m_2}\sum_{i=1}^{m_2}g^{(2)}(z_i) - \frac1{m_1}\sum_{i=1}^{m_1}g^{(2)}(z_i)\right),
\end{align*}
where $A\in\R^{d\times d}$. Let $C_{ij}\in\R^{d\times d}=\Cov[g^{(i)}(Z),g^{(j)}(Z)]$. Then, the autocovariance of $\hat c_{XY}^{\rm MF}$ is given by
\begin{equation}\label{eq: A autocov}
    \Cov[\hat c_{XY}^{\rm MF},\hat c_{XY}^{\rm MF}] 
    =\frac1{m_1} \Gamma_{11} + \left(\frac1{m_1} - \frac1{m_2}\right) \left( A \Gamma_{22}A^\top - \Gamma_{12}A^\top -A \Gamma_{21}\right),
\end{equation}
We now show that the optimal matrix coefficient~\eqref{eq: optimal matrix alpha} minimizes the eigenvalues of~\eqref{eq: A autocov}.

\begin{lemma}\label{lem: optimal A for vectors}
    The choice $A^* = \Gamma_{12}\Gamma_{22}^{-1}$ minimizes all eigenvalues of~\eqref{eq: A autocov}.
\end{lemma}

The following proof follows the argument from~\cite{rubinstein1985efficiency} on multivariate control variates.
\begin{proof}
    Let $D=A-\Gamma_{12}\Gamma_{22}^{-1}$. Then, we can rewrite~\eqref{eq: A autocov} as
    \begin{align*}
    \Cov[\hat c_{XY}^{\rm MF},\hat c_{XY}^{\rm MF}] 
        &=
        \frac1{m_1}\Gamma_{11} + \left(\frac1{m_1} - \frac1{m_2}\right)\left(D(\Gamma_{22})^{-1}D^\top  - \Gamma_{12}(\Gamma_{22})^{-1}\Gamma_{21}\right)\\
        &= \left(\frac1{m_1}\Gamma_{11} - \left(\frac1{m_1} - \frac1{m_2}\right)\Gamma_{12}(\Gamma_{22})^{-1}\Gamma_{21}\right) + \left(\frac1{m_1} - \frac1{m_2}\right) D(\Gamma_{22})^{-1}D^\top.
    \end{align*}
    Denote the first parenthetical term by $S_1 \equiv \left(\frac1{m_1}\Gamma_{11} - \left(\frac1{m_1} - \frac1{m_2}\right)\Gamma_{12}(\Gamma_{22})^{-1}\Gamma_{21}\right)$, and note that this term is symmetric and positive definite: grouping $1/{m_1}$ terms yields the conditional covariance of $g^{(1)}$ given $g^{(2)}$ which must be positive semi-definite and the $1/{m_2}$ term is positive. The second term above is also symmetric positive definite because it has the form $BB^\top$.
    Invoking Weyl's inequality on the eigenvalues of Hermitian matrices (see \cite[Corollary~4.3.12]{horn2012matrix}), we have
    \begin{align*}
        \lambda_i(S_1) \leq \lambda_i\left(S_1 +\left(\frac1{m_1} - \frac1{m_2}\right) D(\Gamma_{22})^{-1}D^\top\right), \quad\text{for }i = 1,2,\ldots,d,
    \end{align*}
    where $\lambda_i(H)$ denotes the $i$-th eigenvalue of the Hermitian matrix $H$.
    Thus, each eigenvalue of $\Cov[\hat c_{XY}^{\rm MF},\hat c_{XY}^{\rm MF}]$ is minimized for the choice $D=0$, and the conclusion follows.
\end{proof}

\begin{theorem}\label{thm: optimal matrix alpha}
    Let $A^*$ be given by~\Cref{lem: optimal A for vectors}. Then, 
    \begin{enumerate}
        \item the generalized variance of the estimator $\hat c_{XY}^{\rm MF}$ is minimized, 
        \item the generalized variance of the estimated regression coefficients $\hat \beta^{\rm MF}$ is also minimized, and 
        \item the variance of the model predictions conditioned on the input is also minimized. 
    \end{enumerate}
\end{theorem}
\begin{proof}
    The first statement follows directly from~\Cref{lem: optimal A for vectors} since minimizing all eigenvalues of a matrix also minimizes its trace. To show the second statement, note that
    \begin{align*}
        \Cov[\hat\beta^{\rm MF},\hat\beta^{\rm MF}] = (C_{XX})^{-1}\Cov[\hat c_{XY}^{\rm MF},\hat c_{XY}^{\rm MF}](C_{XX})^{-\top},
    \end{align*}
    so if $\Cov[\hat c_{XY}^{\rm MF},\hat c_{XY}^{\rm MF}]$ has eigendecomposition $\Cov[\hat c_{XY}^{\rm MF},\hat c_{XY}^{\rm MF}]= V\Lambda V^\top$, then $\Cov[\hat\beta^{\rm MF},\hat\beta^{\rm MF}]$ has eigendecomposition $\Cov[\hat c_{XY}^{\rm MF},\hat c_{XY}^{\rm MF}]=\tilde V\Lambda \tilde V^\top$ where $\tilde V = (C_{XX})^{-1}V$. From \Cref{lem: optimal A for vectors}, each eigenvalue of $\Cov[\hat c_{XY}^{\rm MF},\hat c_{XY}^{\rm MF}]$ is minimized, so each eigenvalue of $\Cov[\hat\beta^{\rm MF},\hat\beta^{\rm MF}]$ is also minimized, and thus the generalized variance $\Tr(\Cov[\hat\beta^{\rm MF},\hat\beta^{\rm MF}])$ is minimized. The third statement follows from the equality in~\Cref{lem: variance upper bound} and~\Cref{lem: optimal A for vectors}.
\end{proof}

In the case where $A = \alpha I$, the covariance expression~\eqref{eq: A autocov} becomes:
\begin{align}\label{eq: alpha autocov}
    \Cov[\hat c_{XY}^{\rm MF},\hat c_{XY}^{\rm MF}] 
    &=\frac1{m_1} \Gamma_{11} + \left(\frac1{m_1} - \frac1{m_2}\right) \left( \alpha^2 \Gamma_{22} - \alpha \Gamma_{12} -\alpha \Gamma_{21}\right),
\end{align}
We now show that the optimal scalar coefficient in \eqref{eq: optimal scalar alpha} minimizes the trace of~\eqref{eq: alpha autocov}.

\begin{lemma}\label{lem: optimal alpha for vectors}
    The choice $\alpha^* = \Tr(\Gamma_{12})/\Tr(\Gamma_{22})$ minimizes the expression for $\Tr(\Cov[\hat c_{XY}^{\rm MF},\hat c_{XY}^{\rm MF}] )$ in~\eqref{eq: alpha autocov}.
\end{lemma}
\begin{proof}
    The trace of the covariance of~\eqref{eq: alpha autocov} is given by (recall $\Tr(M) = \Tr(M^\top)$):
    \begin{align*}
        \Tr( \Cov[\hat c_{XY}^{\rm MF},\hat c_{XY}^{\rm MF}] ) = \frac1{m_1}\Tr(\Gamma_{11}) + \left(\frac1{m_1} - \frac1{m_2}\right)\left(\alpha^2\Tr(\Gamma_{22}) - 2\alpha\Tr(\Gamma_{12})\right).
    \end{align*}
    Setting the derivative of this with respect to $\alpha$ to zero yields:
    \begin{align*}
        2\alpha\Tr(\Gamma_{22}) = 2 \Tr(\Gamma_{21}) \implies \alpha = \Tr(\Gamma_{21})/\Tr(\Gamma_{22}).
    \end{align*}
\end{proof}

\begin{theorem}\label{thm: optimal scalar alpha}
    Let $\alpha^*$ be given by~\Cref{lem: optimal alpha for vectors} and let $A = \alpha^*I$. Then, within all possible control variate coefficients of the form $A=\alpha I$,
    \begin{enumerate}
        \item the trace of the covariance of the estimator $\hat c_{XY}$ is minimized, 
        and 
        \item the upper bound from~\Cref{lem: variance upper bound} on the variance of the model predictions conditioned on the input is also minimized.
    \end{enumerate}
\end{theorem}
\begin{proof}
    The first statement follows directly from~\Cref{lem: optimal alpha for vectors} and the second statement follows from the first statement and~\Cref{lem: variance upper bound}.
\end{proof}

\section{Numerical results}\label{sec: numerics}
In this section, we present numerical results that demonstrate the efficacy of the proposed multifidelity linear regression method applied to an analytical example in \Cref{subsec:analytic} and a PDE model problem in \Cref{subsec: PDE}. We compare models learned with our proposed multifidelity linear regression approach (MF) to models trained in the standard way on only high-fidelity data (HF). To ensure fair comparisons, we compare models with equivalent training cost.

\subsection{Analytic example: exponential function}\label{subsec:analytic}
We demonstrate the multifidelity linear regression method for approximating an exponential function. The analytic expressions for the two fidelities used for the multifidelity method are
\begin{align}
    f^{(1)}(z) &= \exp{z}, \\
    f^{(2)}(z) &= 0.9\sqrt{f^{(1)}(z)} = 0.9 \times \exp(0.5z),
\end{align}
where $z \sim \mathcal{U}(0,5)$ with assumed costs of $w_1 = 1$ and $w_2=0.001$. For this analytical example, the correlation coefficient between the models is $\rho_{12}=0.97$. We fit a fourth-order polynomial model to approximate $f^{(1)}(z)$, which leads to learning five regression coefficients for the one-dimensional problem. That is, $\hat{f}(z;\beta) = \beta_1 + \beta_2z+\beta_3z^2 + \beta_4z^3 + \beta_5z^4= x(z)^\top\beta$ where $x(z) = [1, z, z^2, z^3, z^4]^\top$. We show results for both the standard high-fidelity only and the proposed multifidelity approaches for three different computational budgets of 10, 100, and 1000.

\textit{Experiments using exact model statistics.} For this analytical example, the exact model statistics $\sigma_i$, $\rho_{12}$, and $\Gamma_{12}$ and $\Gamma_{22}$ can be calculated. For a given budget, we use the exact statistics to compute the sample allocations according to~\cref{eq: MFMC optimal allocation}. The resulting sample allocation is shown in \Cref{tab:exp_allocation}. 
\begin{table}[ht]
  \centering
  \begin{tabular}{ccc}
    \toprule
    Computational budget & $m_1$ & $m_2$  \\
    \midrule
    10 & 8 & 1126 \\
    100 & 88 & 11263 \\
    1000 & 887  & 112631 \\
    \bottomrule
  \end{tabular}
  \caption{Sample allocations for different computational budgets based on~\eqref{eq: MFMC optimal allocation} assuming exact model statistics for the analytic example.}
  \label{tab:exp_allocation}
\end{table}
Based on these exact statistics, the heuristic scalar control variate coefficient~\eqref{eq: MFMC mean} is $\alpha=12.79$; the optimal scalar control variate coefficient is $\alpha^*=11.97$; and the optimal matrix control variate coefficient is 
\begin{align*}
    A^* = \begin{pmatrix}
        1.8\text{e}1 & -8.0\text{e}0 & 2.2\text{e}0 & -2.7\text{e-}1 & 2.4\text{e-}2 \\ 
        2.2\text{e}2 & -1.1\text{e}2 & 2.7\text{e}1 &-3.6\text{e}0 & 2.9\text{e-}1\\
        2.5\text{e}3 & -1.3\text{e}3 & 3.0\text{e}2 & -4.2\text{e}1 & 3.1\text{e}0 \\
        2.6\text{e}4 & -1.3\text{e}4 & 3.1\text{e}3 & -4.2\text{e}2 & 2.9\text{e}1 \\
        2.3\text{e}5 & -1.1\text{e}5 &2.7\text{e}4 & -3.7\text{e}3 & 2.4\text{e}2
    \end{pmatrix}.
\end{align*}

To facilitate comparison between the standard high-fidelity only (HF) training approach and our proposed multifidelity (MF) training approach, we first focus on multifidelity results using the optimal scalar coefficient $\alpha^*$, which is our general recommendation. We repeat the model learning process over 500 independent realizations of training data to analyze robustness of the proposed method to variations in training data.
In \Cref{fig:exp_unbiased}, the left subplot plots realizations of the first element of $\hat c_{XY}$, the central subplot realizations of the first regression coefficient $\hat \beta_1$, and the right subplot realizations of the regression model prediction at $z=5$, $\hat f(z=5;\hat\beta)$. The replicate mean over the 500 training realizations is also shown, along with the truth that would be obtained from taking exact expectations. The realizations are plotted with some transparency so that regions with more realizations appear more densely colored. These numerical results illustrate that the multi-fidelity estimators we have introduced are unbiased (\Cref{thm: unbiased}), and that the multifidelity estimators exhibit lower variance than their high-fidelity counterparts at all computational budgets. 
  
\begin{figure}[!htb]
  \centering
    \includegraphics{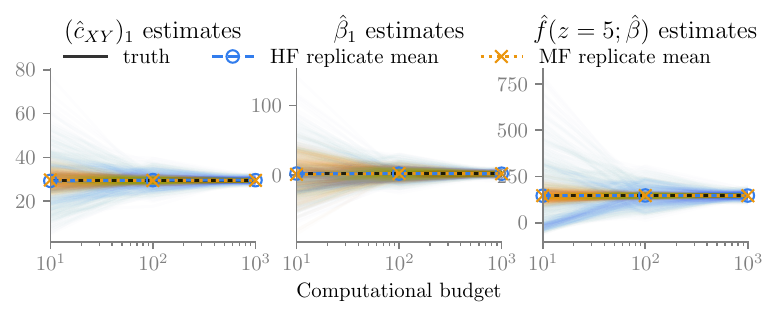}
  \caption{Exponential function example: 500 realizations (semi-transparent colored lines) of estimators for first entry of $\hat{c}_{XY}$ (left), first regression coefficient (center), and regression model prediction $\hat{f}(z=5;\hat{\beta})$ (right), when true model statistics are known. Black lines denote true values and lines with markers show the mean over the 500 realizations of training data.}
  \label{fig:exp_unbiased}
\end{figure}

\Cref{fig:exp-exact-comp} illustrates that the variance reduction achieved by our multifidelity training approach yields more accurate and robust learned models. 
\Cref{fig:exp_regcomp} plots realizations of the fitted regression models $\hat f(z;\hat \beta)$ using both the standard high-fidelity only (HF) and multifidelity (MF) training approaches for a computational budget of 100, together with a comparison to the exact regression model $\hat f(z;\beta^*)$, which can be computed for this analytical example. Note that both HF and MF learned models are unbiased estimators for the exact regression model, but that the HF approach exhibits higher variance and is less robust to variations of the training data, even with the relatively high budget of 100 high-fidelity samples for this simple example. 
\Cref{fig:exp_gen} plots the generalization error of the learned regression models trained with the HF and MF approaches at different computational budgets. For this experiment, 1000 random test data are generated from the input distribution for each learned model replicate, and the mean relative error of the learned model, $\left|\frac{\hat f(z;\hat\beta) - f^{(1)}(z)}{f^{(1)}(z)}\right|$, is evaluated for 500 realizations of the HF and MF learned models. The mean and first standard deviations over these 500 realizations are shown. 
For this example, the MF training has a similar mean generalization error than the standard HF-only training using the same computational budget, but a significantly reduced variance: the high-fidelity only training can lead to much higher generalization errors at low computational budgets. This illustrates the power of our multifidelity training strategy to learn more robust models from scarce high-fidelity data.

\begin{figure}
    \centering
    \begin{subfigure}[t]{0.48\textwidth}
        \centering
        \includegraphics{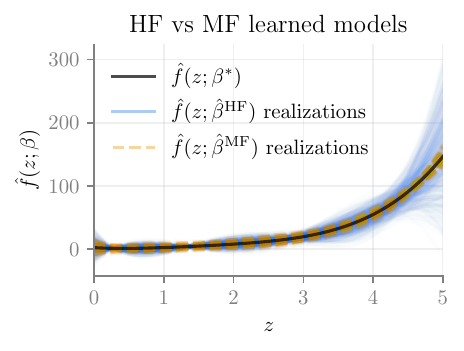}
        \caption{Visualization of 500 HF and MF learned models trained from independent realizations of training data with a computational budget of 100.}
        \label{fig:exp_regcomp}
    \end{subfigure}
    \hfill
    \begin{subfigure}[t]{0.48\textwidth}
        \centering
        \includegraphics{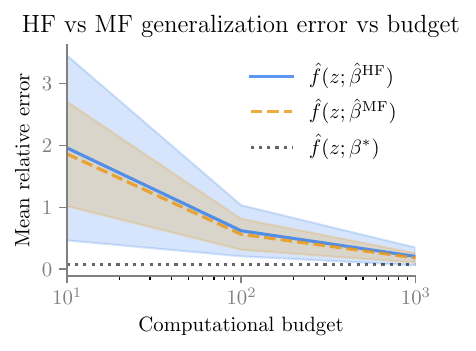}
        \caption{Generalization error over 1000 unseen test data. Plotted lines and shaded regions are the mean and first standard deviation over 500 independently trained models.}
        \label{fig:exp_gen}
    \end{subfigure}
    \caption{Exponential example: comparing learned models learned with the standard high-fidelity only (HF) and proposed multifidelity (MF) training approach. }
    \label{fig:exp-exact-comp}
\end{figure}

\textit{Experiments using inexact model statistics. }
We now explore the effect of estimating the statistics required to set the control variate coefficient $A$ and determine the model allocation in our multifidelity approach, and compare the performance of different control variate coefficient choices when these statistic estimates vary in quality. 
\Cref{fig:exp_conv} plots the (generalized) variance of estimates of $\hat c_{XY}$ (top row), $\hat\beta$ (middle row), and $\hat f(z;\hat\beta)$ (bottom row) over 500 realizations of training data. In each subplot, the standard high-fidelity only approach is compared to the proposed multifidelity approach with the three choices of control variate coefficient discussed in \Cref{subsec: CV coefficients}. The columns of \Cref{fig:exp_conv} correspond to using statistics estimated exactly (left), and using 100 (middle), and 10 (right) pilot samples. Estimated statistics are used to determine both the control variate coefficient and the model allocation: we find that using estimated model statistics to determine the sample allocation according to~\eqref{eq: MFMC optimal allocation} yields a 1\% (5\%) variation in $m_1$ and a 3\% (25\%) variation in $m_2$ when 100 (10) pilot samples are used.

\begin{figure}[!htb]
  \centering
 \includegraphics[width=0.8\textwidth]{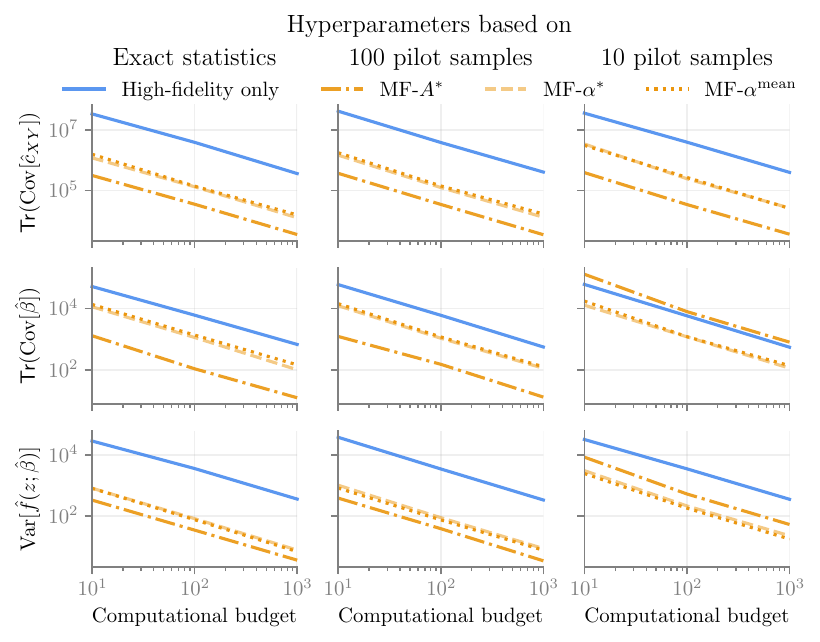}
  \caption{Analytical example: convergence of multifidelity linear regression estimators for $\hat c_{XY}$ (top), $\hat\beta$ (middle), and $\hat f(z;\hat\beta)$ (bottom) when model statistics are exact (left), estimated using 100 pilot samples (center), or 10 pilot samples (right).}
  \label{fig:exp_conv}
\end{figure}

Our results in \Cref{fig:exp_conv} show that for this analytical example, the multifidelity estimators generally have lower variance than their standard single-fidelity counterparts, which leads to more robust models with lower average errors, as discussed above. 
These results therefore demonstrate that our proposed multifidelity training approach consistently leads to more robust and accurate learned models for this analytical example.
\Cref{fig:exp_conv} also illustrates our theoretical results from \Cref{subsec: analysis of MFCE} concerning the optimal choices of control variate coefficient: when exact statistics are known, the optimal matrix choice exhibits the greatest variance reduction, while the scalar choices, MF-$\alpha^*$ and MF-$\alpha^{\rm mean}$, exhibit similar performance. For this example, the optimal matrix choice remains the most effective even when statistics are estimated from 100 pilot samples, but when statistics are estimated using just 10 pilot samples this `optimal' choice leads to sub-optimal variance reduction in the regressor predictions and in fact higher variance in the learned parameters than the high-fidelity only training achieves. This underscores our general recommendation to use the optimal scalar control variate coefficient in practical settings where the model statistics are unknown and must be estimated, although we note that for this example the heuristic scalar choice has similar performance. In fact, the heuristic choice slightly outperforms the `optimal' scalar choice when model statistics are estimated from 10 pilot samples for this analytical example, demonstrating that this may also be a reasonable choice in practice. 
With either scalar control variate coefficient choice, the multifidelity models trained with the lowest budget (just 8 high-fidelity samples) have similar prediction variance to standard learned models trained with hundreds of high-fidelity samples, demonstrating the efficacy of the multifidelity approach for learning robust and accurate models in the scarce data regime.

\subsection{A PDE model problem}\label{subsec: PDE}
We now demonstrate our multifidelity linear regression method on a two-dimensional hydrogen combustion model with five parameters. The high- and low-fidelity models were originally developed in~\cite{buffoni2010projection} and then used for sensitivity analysis in~\cite{qian2018multifidelity}. We follow the setup in~\cite{qian2018multifidelity}, which we summarize in \Cref{subsec: CDR problem}. \Cref{subsec: CDR results} then presents results for this problem.

\subsubsection{Convection-diffusion-reaction (CDR) problem}\label{subsec: CDR problem}

We consider a convection-diffusion-reaction model on a two-dimensional rectangular domain. The model assumes a premixed hydrogen flame at constant, uniform pressure, with constant, divergence-free velocity field, and equal, uniform molecular diffusivities for all species and temperatures. The dynamics of the system are described by a convection-diffusion-reaction equation with source terms modeled as in \cite{cuenot1996asymptotic} as follows:
\begin{equation} \label{eq:CDR}
\begin{split}
  \frac{\partial u}{\partial t} &= \kappa\Delta u -U_{\rm vel}\nabla u +s(u,z)\\
   s_i(u,z)&=\gamma_i\nu_i\left(\frac{W_i}{\rho_\text{mix}}\right)\left(\frac{\rho_\text{mix} Y_F}{W_F}\right)^{\nu_F}\left(\frac{\rho_\text{mix} Y_O}{W_O}\right)^{\nu_O}A_\text{pe}\exp\left(-\frac{E}{RT}\right),\,\quad \text{for } i=F,O,P,\\
  s_T(u,z) &= s_P(u,z)Q.
\end{split}
\end{equation}
In \eqref{eq:CDR}, the thermo-chemical state is $u(z,t)=[Y_F,Y_O,Y_P,T]^\top$, where $Y_F$, $Y_O$, and $Y_P$ are the mass fractions of the fuel, oxidizer, and product, respectively, and $T$ is the temperature; $\kappa$ is the molecular diffusivity, $U_{\rm vel}$ is the velocity field, and $s$ is the nonlinear reaction source term. Additionally, $\gamma_P=1$ and $\gamma_O=\gamma_F=-1$.
The reaction modeled is a one-step hydrogen combustion given by $2\text{H}_2+\text{O}_2\rightarrow 2\text{H}_2\text{O}$, where hydrogen is the fuel, oxygen acts as the oxidizer, and water is the product. 
The source terms are defined using $A_\text{pe}$ as the pre-exponential factor of the Arrhenius equation, $E$ as the activation energy, $W_i$ as the molecular weight of species $i$, $\rho_\text{mix}$ as the density of the mixture, $R$ as the universal gas constant, and $Q$ as the heat of the reaction. 

For our regression problem, the quantity of interest is the maximum temperature in the chamber at steady state. The input parameter vector for regression is given by 
\begin{equation*}
    s=[A_\text{pe},E,T_i,T_0,\phi] \in [5.5\times 10^{11},1.5\times 10^{12}]\times [1.5\times 10^3,9.5\times10^3]\times[200,400]\times[850,1000]\times[0.5,1.5],
\end{equation*}
where $T_i$ is the temperature at the inlet, $T_0$ is the temperature of the left wall of the domain, and $\phi$ is the fuel:oxidizer ratio of the premixed inflow. The parameters $A_\text{pe}$ and $E$ are assumed to be log-uniformly distributed and $T_i$, $T_0$, and $\phi$ are assumed to be uniformly distributed. We fit a quadratic model to these five inputs, so that the feature vector $x$ has length 21 (one constant, five linear, and 15 quadratic features). Due to the varying scales of the input parameters, the parameters are scaled to be order 1 before monomial features are computed.

The high-fidelity data arises from a finite-difference solver that discretizes the spatial domain into a $73\times37$ grid, resulting in $73\times37\times4=10804$ degrees of freedom that includes the mass fractions of each of the three chemical species as well as the temperature at each grid point. The low-fidelity data arises from a POD-DEIM projection-based reduced model~\cite{chaturantabut2010nonlinear} with 19 POD basis functions and 1 DEIM basis function (see \cite{buffoni2010projection} for details of the model reduction approach applied to this reacting flow problem). Statistics for these two models computed using a data set consisting of $2.4\times 10^5$ samples are presented in \Cref{tab: CDR models}~\cite{qian2018multifidelity}. This data set can be downloaded from \url{https://github.com/elizqian/mfgsa}. For all numerical experiments below, the covariance $C_{XX}$ is computed from all $2.4\times 10^5$ samples. The standard and multifidelity regression estimators for $\hat c_{XY}$ are computed by bootstrapping samples from this data set within budget constraints, assuming that high-fidelity data have cost $w_1=1.94$ and low-fidelity data have cost $w_2 = 6.2\text{e-}3$ as stated in~\cite{qian2018multifidelity}.
\begin{table}[ht]
  \centering
  \begin{tabular}{ll |l l l l}
    model & & $\mu_i$ & $\sigma_i$ & $\rho_{1i}$ & $w_i$ \\
    \hline
    High-fidelity (FD) & $f^{(1)}$
    & 1406 & 276 & 1 & 1.94\\
    Low-fidelity (POD-DEIM) & $f^{(2)}$ &
    1349 & 356 & 0.95 & 6.2e-3
  \end{tabular}
  \caption{High-sample model statistics for the high- and low-fidelity models for the CDR problem using $2.4\times 10^5$ samples~\cite{qian2018multifidelity}.}
  \label{tab: CDR models}
\end{table}

\subsubsection{Results: CDR problem}\label{subsec: CDR results}
Because analytical model statistics are not easily obtainable for this example, we compute `exact' reference statistics using all $2.4\times 10^5$ samples in the data set from which we bootstrap the training data. \Cref{fig:cdr_conv} (analogously to \Cref{fig:exp_conv}) plots the (generalized) variance of estimates of $\hat c_{XY}$ (top row), $\hat\beta$ (middle row), and $\hat f(z;\hat\beta)$ (bottom row) over 500 realizations of training data for computational budgets of 10, 100, and 1000. The columns of \Cref{fig:cdr_conv} correspond to using all $2.4\times 10^5$ samples (left), 100 pilot samples (center), and 10 pilot samples (right) to estimate model statistics that are then used to determine the control variate coefficients and sample allocations in our multifidelity approaches. \Cref{tab:cdr2d_allocation} provides the sample allocation obtained using all $2.4\times 10^5$ samples. We find that using fewer pilot samples to determine the sample allocation according to~\eqref{eq: MFMC optimal allocation} yields a 1\% (10\%) variation in $m_1$ and a 6\% (30\%) variation in $m_2$ when 100 (10) pilot samples are used.

\begin{table}[ht]
  \centering
  \begin{tabular}{ccc}
    \toprule
    Computational budget & $m_1$ & $m_2$  \\
    \midrule
    10 & 4 & 250 \\
    100 & 43 & 2504 \\
    1000 & 435 & 25045 \\
    \bottomrule
  \end{tabular}
  \caption{Sample allocation based on~\eqref{eq: MFMC optimal allocation} for the CDR problem based on reference model statistics computed using all available samples in the data set.}
  \label{tab:cdr2d_allocation}
\end{table}

\begin{figure}[!htb]
  \centering
 \includegraphics[width=0.8\textwidth]{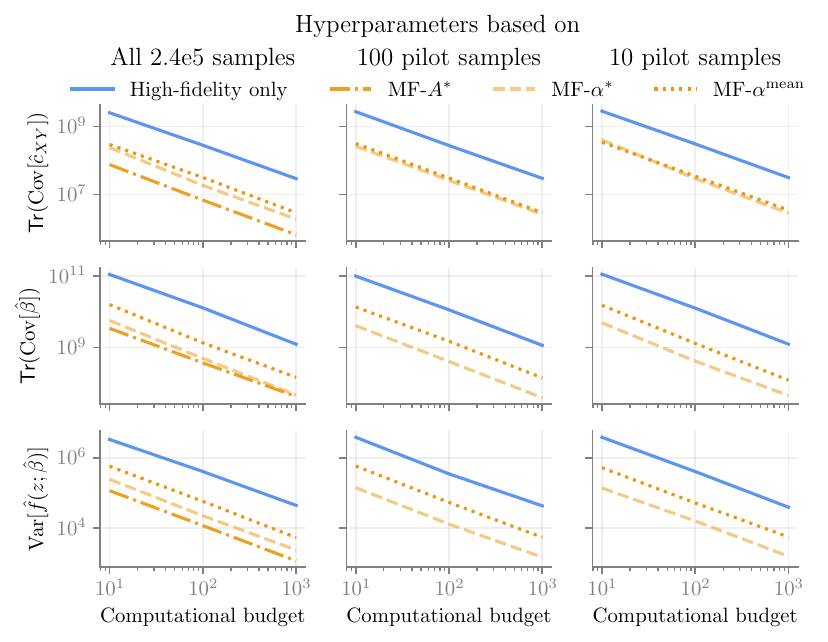}
  \caption{PDE model problem: convergence of multifidelity linear regression estimators for $\hat c_{XY}$ (top), $\hat\beta$ (middle), and $\hat f(z;\hat\beta)$ (bottom), when model statistics are estimated using $10^5$ (left), 100 (center), or 10 (right) pilot samples. Results for the multifidelity approach with the optimal matrix coefficient are omitted in the second and third columns because their variances are so large that they would significantly distort the plot axes.}
  \label{fig:cdr_conv}
\end{figure}

For this PDE model problem, \Cref{fig:cdr_conv} shows that when the model statistics are estimated using many samples, the optimal matrix control variate coefficient leads to the greatest variance reduction. However, when only 100 or 10 pilot samples are available, the `optimal' coefficient actually leads to much higher multifidelity variance than the standard high-fidelity only approach, so we have chosen to omit the MF-$A^*$ lines from the second and third columns because they would distort the axes significantly. However, both choices of scalar control variate coefficient consistently lead to multifidelity linear regression estimators with lower variance than their high-fidelity counterparts, and the optimal scalar coefficient consistently performs better than the heuristic choice. For this reason, our general recommendation is to use the optimal scalar choice, although for complex engineering problems it may be easier to use the heuristic choice, and our numerical results illustrate that this still leads to practical gains: for this problem, both scalar coefficient choices yield multifidelity learned models that, when trained with the lowest computational budget (just 4 high-fidelity samples), achieve a similar average error to standard learned models trained with $\mathcal{O}(10-100)$ times as many high-fidelity samples.

Finally, we plot in \Cref{fig:cdr_gen} the mean generalization error of the HF and MF learned models over 1000 test data bootstrapped independently from the available data set. The mean and first standard deviation over 500 learned models trained on independent realizations of training and test data are shown. The multifidelity results use the optimal scalar coefficient, our general recommendation. \Cref{fig:cdr_gen} shows that our MF training approach achieves a mean relative error that is several times less than that of the standard HF only training approach. The error reduction is especially stark at our lowest computational budget, again demonstrating the efficacy of the multifidelity approach for learning robust and accurate models in the scarce data regime. 

\begin{figure}[!htb]
  \centering
 \includegraphics{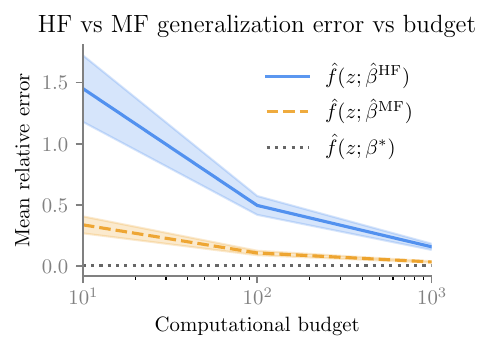}
  \caption{Generalization error over 1000 unseen test data. Plotted lines and shaded regions are the mean and first standard deviation over 500 learned models trained on independently realizations of training data.}
  \label{fig:cdr_gen}
\end{figure}

\section{Conclusions and future work}\label{sec: conclusions}
We have introduced a multifidelity training strategy for linear regression problems which leverages data of varying fidelity and cost to improve the robustness of training to scarce high-fidelity data due to limited training budgets. Our approach is based on using multifidelity control variate estimators to reduce the variance of the linear regression estimators. We propose new multifidelity Monte Carlo estimators for the linear regression problem, which we show are unbiased with respect to the high-fidelity data distribution. We discussed multiple strategies for choosing the amount of data for each fidelity level, and for choosing the control variate coefficients in the multifidelity estimators, which can be viewed as hyperparameters of the method. We provide theoretical analyses showing that there are optimal choices of control variate coefficients when model statistics are known, and provide practical recommendations for choosing the hyperparameters when model statistics are unknown and must be estimated form limited samples. Numerical experiments illustrate our theoretical results and also demonstrate that the multifidelity training approach learns more accurate and robust linear regression models than the standard high-fidelity only training approach when the training budget is limited and only scarce high-fidelity data are available. 

Many directions exist for future work. Within the linear regression setting we have introduced here, several avenues for further theoretical and practical investigation remain, including analysis of the method to determine optimal model selection and budget allocations tailored to the linear regression problem (instead of using the heuristic based on mean estimation that we recommend here), and exploring hybrid or active data acquisition strategies where some data are given but more can be obtained within budget constraints. Detailed analysis and experiments comparing the MFMC-based estimator we have presented and alternatives based on ACV and MLBLUE would also be worthwhile. There are also many interesting questions to explore in using multifidelity control variates beyond the linear regression setting to more complex models, including neural networks. While some existing works have considered aspects of these problems~\cite{gorodetsky2020mfnets,de2020transfer,gerstner2021multilevel}, many theoretical and methodological questions remain. 


\textit{Acknowledgments.}
The authors are grateful to Florian Sch\"afer and Rob Webber for helpful discussions, and to the anonymous referees for helpful feedback that has improved the clarity and completeness of the work. 
Work by EQ was supported by the US Department of Energy Office of Science Energy Earthshot Initiative as part of the `Learning reduced models under extreme data conditions for design and rapid decision-making in complex systems' project under award number DE-SC0024721. AC acknowledges support from the US Department of Energy (DOE) grant number DE-SC0021239. AC and VS acknowledge support from DARPA Automating Scientific Knowledge Extraction and Modeling (ASKEM) program under Contract No.\ HR0011262087 (award number 653002). The views, opinions, and/or findings expressed are those of the author(s) and should not be interpreted as representing the official views or policies of the Department of Defense or the U.S.~Government.

\appendix
\section{Alternative control variate estimators for the linear regression problem}\label{app: other CV estimators}
The main idea of this work is to propose a multifidelity approach to training linear regression models that combines high- and low-fidelity data using control variate estimators, thereby reducing the variance of the learned model predictions and increasing accuracy and robustness. In the main text, our focus is on the multifidelity learned model~\eqref{eq: MF CEA}, which we describe as \textit{covariance estimation} since the definitions of $\hat\beta^{\rm MF}$ and the associated learned model follow from estimates of $\hat c_{XY}^{\rm MF}$, as summarized here:
\begin{align*}
    \hat f^{\rm MF}(z) = \hat f(z;\hat\beta^{\rm MF}), \quad \hat\beta^{\rm MF} = C_{XX}^{-1}\hat c_{XY}^{\rm MF},
\end{align*}
and $\hat c_{XY}^{\rm MF}$ is given by~\eqref{eq: cov estimate A}. One could, however, directly formulate a control variate estimator for the unknown parameters $\beta$ and use this to define a model. We describe such an approach as \textit{parameter estimation}:
\begin{align}\label{eq: param estimation CV}
    \hat f^{\rm MF}_{\rm param}(z) &= \hat f(z;\hat\beta^{\rm MF}_{\rm param}), \qquad
    \hat\beta^{\rm MF}_{\rm param} = \hat\beta_{m_1}^{(1)} + \tilde A (\hat\beta_{m_2}^{(2)} - \hat\beta_{m_1}^{(2)}),
\end{align}
where $\hat\beta_{m_i}^{(j)}=\frac1{m_i}C_{XX}^{-1}X_{m_i}Y_{m_i}^{(j)}$ denotes the parameter estimator obtained from $m_i$ samples of data from model $j$, and $\tilde A\in\R^{d\times d}$.
Finally, one could directly formulate a control variate estimator for the unknown model itself: we describe this approach as \textit{model estimation}:
\begin{align}\label{eq: model estimation CV}
    \hat f^{\rm MF}_{\rm model} = \hat f(z;\hat\beta_{m_1}^{(1)}) + \tilde\alpha \left(\hat f(z;\hat\beta_{m_2}^{(2)})-\hat f(z;\hat\beta_{m_1}^{(2)})\right),
\end{align}
where we note $\tilde\alpha\in\R$ is a scalar since the output of $\hat f$ is one-dimensional. 

The covariance, parameter, and model estimation approaches are closely related and are equivalent in some cases. In particular, if we let $A =\tilde A = \tilde\alpha I$, then all three estimators are the same. Additionally, the parameter estimation approach~\eqref{eq: param estimation CV} is equivalent to the covariance approach if $\tilde A = C_{XX}^{-1} A C_{XX}$; that is, the parameter estimation approach represents a change of coordinates (via $C_{XX}^{-1}$) of the covariance estimation approach which is the focus of the main text. 

We note that the model~\eqref{eq: model estimation CV} and parameter~\eqref{eq: param estimation CV} estimation approaches can also be analyzed in a similar vein to our analysis of \Cref{subsec: analysis of MFCE}. The optimal matrix coefficient for the parameter estimator~\eqref{eq: param estimation CV} will be related to the optimal matrix coefficient for the covariance estimator by a similarity transform as discussed above, and will thus yield the same variance of parameters and model predictions. However, the optimal scalar coefficients for the three estimation approaches will generally be different, and will depend on second-order statistics of $C_{XX}^{-1}x(Z)f^{(i)}(Z)$ in the parameter case and of $x^\top(Z)C_{XX}^{-1}x(Z)f^{(i)}(Z)$ in the model case (recall the optimal coefficients depend on second-order statistics of $x(Z)f^{(i)}(Z)$ for the covariance case). When $C_{XX}$ is known exactly, the optimal coefficient for parameter estimation is as computable as the optimal coefficient for covariance estimation. However, when one generalizes to needing to estimate $C_{XX}$ as well, estimating the required second-order statistics for parameter estimation becomes more complicated. Estimating second-order statistics for the model estimation case is also generally more difficult than in the covariance estimation case: the model estimation setting essentially requires second-order statistics for the exact regression models themselves. We therefore have restricted our focus to the covariance estimation case, noting that further numerical and theoretical exploration of the alternative estimators discussed in this appendix is an area for future work.


\bibliographystyle{siam}
\bibliography{qian_career,refs}

\end{document}